\documentclass[twoside]{article}

\usepackage{xcolor}
\usepackage{amsmath}
\usepackage{amsfonts}
\usepackage{mathtools}
\usepackage{xr-hyper}
\usepackage{hyperref}
\usepackage{amsthm}

\newcommand{\calF}{\mathcal{F}}
\newcommand{\bbP}{\mathbb{P}}
\newcommand{\bbR}{\mathbb{R}}
\newcommand{\bbN}{\mathbb{N}}
\newcommand{\bbE}{\mathbb{E}}

\newcommand{\calX}{\mathcal{X}}
\newcommand{\calB}{\mathcal{B}}
\newcommand{\calP}{\mathcal{P}}
\newcommand{\calQ}{\mathcal{Q}}
\newcommand{\calS}{\mathcal{S}}
\newcommand{\calN}{\mathcal{N}}
\newcommand{\calO}{\mathcal{O}}
\newcommand{\calM}{\mathcal{M}}
\DeclareMathOperator*{\argmin}{arg\,min}

\newtheorem{lemma}{Lemma}
\newtheorem{theorem}[lemma]{Theorem}

\usepackage[accepted]{aistats2022}
\usepackage[round]{natbib}

\begin{document}

%

%
\runningauthor{Wild, Wynne}
\twocolumn[

\aistatstitle{Variational Gaussian Processes: A Functional Analysis View}

\aistatsauthor{Veit Wild* \And George Wynne*}

\aistatsaddress{University of Oxford \And  Imperial College London} ]

\begin{abstract}
  Variational Gaussian process (GP) approximations have become a standard tool in fast GP inference. This technique requires a user to select variational features to increase efficiency. So far the common choices in the literature are disparate and lacking generality. We propose to view the GP as lying in a Banach space which then facilitates a unified perspective. This is used to understand the relationship between existing features and to draw a connection between kernel ridge regression and variational GP approximations.
\end{abstract}

\section{Introduction}\label{sec:intro}
Gaussian processes (GPs) are a ubiquitous modelling paradigm within machine learning \citep{Rasmussen2005}. They are random functions with the useful property that their pointwise evaluations form a multivariate Gaussian random vector. Within the Bayesian framework one may use a Gaussian process as a prior for an unknown function, condition on observed information and then a posterior over the unknown function is obtained. Examples of applications include regression, classification, reinforcement learning and optimisation. See \citet{Rasmussen2005} for an introduction. The GP framework enjoys such popularity since it is flexible, interpretable, has closed form expressions in some scenarios and offers a degree of uncertainty quantification. 

An issue of the GP framework is the naive computational cost $O(N^{3})$ to perform predictions, where $N$ is the number of observed data points. This is due to a matrix inversion term. Many methods have been proposed to reduce this computational cost to something more palatable, both through theoretical \citep{csato2002sparse,seeger2003fast,snelson2006sparse,snelson2007local,titsias2009variational} and computational innovations \citep{gardner2018gpytorch,wang2019exact}. 

The focus of this paper is the variational inference paradigm where the posterior GP is approximated by an element of a candidate family through an optimisation routine where distance from the posterior GP is measured with the Kullback-Leibler (KL) divergence. The common candidate family is a family of GPs formed by conditioning the prior on $M$ surrogate \emph{features}, not necessarily equal to the observed information, resulting in $O(NM^2)$ complexity rather than the aforementioned $O(N^{3})$. For example, features could be point evaluations or values of inner products against some user chosen set of functions. Foundational papers regarding variational inference for GPs include \citet{titsias2009variational,Matthews16} and for a survey consult \citet{leibfried2020tutorial}. 

Choosing features to condition on is often conducted with the aim of closed form, or at least easy to compute, expressions. Therefore the features chosen can be very dependent on the given GP of interest through the covariance kernel, mean function or space the GP takes values in. This has led to a heuristic and somewhat ad-hoc approach in the literature to deriving features. Indeed, there is little in the way of a unified view of the choice of features and how different choices relate to each other.

\textbf{Contributions:} We present a unified perspective of existing features used in variational Gaussian process approximations by embracing the fact that GPs can be viewed as Gaussian random elements in Banach spaces. This perspective  reveals generalisations of, and equivalences between, commonly used variational features. In particular we generalise the derivation for the popular variational Fourier features \citep{Hensman2018}. This injects rigour and clarity into the features used. Finally, a connection to kernel ridge regression is made which clarifies the role that the variational features play in the posterior approximation. 

\textbf{Existing work:} The commonly employed framework of variational GP approximation was derived  by \citet{titsias2009variational}. Two other types of features we focus on are inter-domain \citep{lazaro2009inter} and Fourier \citep{Hensman2018}. A formalism of the use of Kullback-Liebler divergence over infinite dimensions, a key part of the variational GP methodology, was clarified by \citet{Matthews16}. The Fourier features have recently been combined with spherical harmonics \citep{Dutordoir20a} and rough path theory \citep{lemercier2021}.

\section{Gaussian Processes and Gaussian Random Elements}\label{sec:GP_GRE}
Let $(\Omega,\calF,\bbP)$ be a the underlying probability space on which all random quantities are defined. Let $\calX$ be a set. A family of random variables $G\colon \calX\times\Omega\rightarrow\bbR$ defined on $(\Omega,\calF,\bbP)$ is called a random process. Let $G(x)$ denote the random variable $G(x,\cdot)$. A random process is called a \emph{Gaussian process} (GP) if for every $N \in\bbN$ and $\{x_{n}\}_{n=1}^{N}\subset \calX$ the random vector $(G(x_{1}),\ldots,G(x_{N}))$ is Gaussian. A Gaussian process is entirely determined by its mean function $m\colon \calX \rightarrow\bbR$ defined $m(x)\coloneqq\bbE[G(x)]$ and covariance function, also know as covariance kernel, $k\colon \calX \times \calX \rightarrow\bbR$ defined  $k(x,x') \coloneqq \bbE[(G(x)-m(x))(G(x')-m(x'))]$. We denote the GP with mean $m$ and covariance kernel $k$ as $G\sim GP(m,k)$. For background on Gaussian processes consult the works by \citet{Rasmussen2005,Lifshits2012,Adler1990}.

Let $E$ be a separable Banach space, $\calB(E)$ the Borel $\sigma$-algebra and $\calP(E)$ the set of Borel probability measures on $E$. We will always assume that $E$ is separable without explicitly stating it every time. The dual of $E$ is defined as $E^{*}\coloneqq\{x^{*}: E \to \bbR \, | \, x^{*}\text{ is linear and continuous} \}$ and for $x^{*}\in E^{*}, y\in E$ we write $(y,x^{*})_{E}\coloneqq x^{*}(y)$ for the so called \textit{dual pairing}. A mapping $F\colon\Omega\rightarrow E$ is called a \emph{Gaussian random element} (GRE) if for every $x^{*}\in E^{*}$ the real valued random variable $x^{*}(F)=(F,x^*)_E$ is Gaussian. Each GRE has an associated mean $m\in E$ which is uniquely characterised by satisfying $(m,x^*)_E = \bbE[(F,x^*)_E)]$ for all $x^{*}\in E^{*}$ and covariance operator $C\colon E^{*}\rightarrow E$ uniquely characterised by satisfying $(Cx^{*},y^{*})_{E} = \text{Cov}[(F,x^*)_E,(F,y^*)_E]$ for all $x^*,y^* \in E^*$. We denote $F\sim \mathcal{N}(m,C)$ for a GRE with mean $m$ and covariance operator $C$. Note that for $E= \bbR^N$ this coincides with the standard definition for the normal distribution in $\bbR^N$.

A measure $P\in \calP(E)$ is called a \emph{Gaussian measure} (GM) if for every $x^{*}\in E^{*}$ the pushforward measure $P^{x^{*}}\in \calP(\bbR)$ defined by $P^{x^{*}}(\cdot)\coloneqq (P\circ (x^{*})^{-1})(\cdot)$ is a Gaussian measure on $ \mathcal{B}(\bbR)$. As with random variables in $\bbR$ and probability measures on $\mathcal{B}(\bbR)$ there is a one-to-one correspondence between GREs and GMs on $E$. GREs and GMs can be studied in far more generality, or indeed with more specificity, than $E$ being a Banach space \citep{Bogachev1998,DaPrato2006}. 

When a GP $G$ satisfies $\bbP(\{\omega\colon G(\cdot,\omega)\in E\}) = 1$ we say that its sample paths lie almost surely in $E$. This has been studied for numerous common choices of $E$ \citep{Rajput1972,Rajput1972Lp,Lukic2001} and one can then identify the GP with a GRE, or equivalently with a GM, over $E$. Throughout the rest of this paper we shall be dealing purely in terms of GREs and later specific examples equating these to GPs shall be given. The use of GREs facilitates a general view of variational inference and will be the vehicle of our results. 

\subsection{The problem with the path view}
\label{subsec: process view}

This subsection will motivate the use of GREs in the analysis of variational GP approximation methods by arguing why the path view is insufficient. 

Let $G \sim GP(m,k)$ be a Gaussian process on $\calX$ and $\bbR^\calX:=\{f : \, f: \calX \to \bbR \}$ the vector-space of functions from $\calX$ to $\bbR$. Define $\pi_x: \bbR^\calX \to \bbR, \, f \mapsto f(x)$ the pointwise evaluation map for some $x \in \calX$. The $\sigma$-algebra generated by the pointwise evaluations $\{\pi_x\}_{x\in\calX}$, meaning the smallest $\sigma$-algebra such that all projections $\pi_x$ are measurable, is denoted $\mathcal{S}$. Every random process $G$ with paths over $\calX$ can be canonically identified with a random element  in $\bbR^{\calX}$ $F: \big(\Omega, \mathcal{A}, \bbP \big) \to \big( \bbR^\calX, \mathcal{S} \big)$ via $F(\omega) = G(\cdot, \omega)$ \citep[Chapter 3]{kallenberg1997foundations}. In short, a stochastic process is nothing but a random element in the \textit{large} space $\bbR^{\calX}$ with a \textit{small} $\sigma$-algebra $\mathcal{S}$. 

This presents two main issues. First, $\calS$ is so small that while pointwise evaluations are measurable, operations we want to perform in the variational GP framework may not be, such as integration of GPs. Secondly, the theoretical framework of variational GP approximations outlined by \citet{Matthews16} requires the paths to lie in a Polish space so that Bayes theorem may be used. The space $\bbR^{\calX}$ is not a Polish space so is not an appropriate space to view our GP paths in. 

The remedy will be to view the GP paths in a Banach space, which is Polish and has enough structure so that operations we want to perform on the GP are valid. To do this we will equate the GP to a GRE.

\section{Gaussian Random Element Regression}\label{sec:GRER}
In this section we outline Gaussian random element regression  which is the random element view of standard Gaussian process regression. This view is commonly employed in areas such as Bayesian inverse problems \citep{Stuart2010}. At first glance the framework may appear (indulgently) abstract. However, we believe this is the most natural framework to investigate variational GP approximation. There is an important example at the end of the section showing all that follows does in fact coincide with standard GP regression. 

Let $F$ be a GRE in $E$ with mean $m$ and covariance operator $C$ and denote by $P\in\calP(\calX)$ its corresponding GM. Suppose we have observations $Y = \{Y_{n}\}_{n=1}^{N}$ which are the image of $F$ under some $\{D_n\}_{n=1}^{N} \subset E^*$, corrupted by independent scalar Gaussian noise
\begin{align*}
    Y_n = (F,D_n)_E + \epsilon_n,
\end{align*}
where $\epsilon_n \sim \calN(0,\sigma^2)$ independently for $n=1,\ldots,N$. This can be equivalently expressed in perhaps more familiar notation as the probability density function (pdf) of $Y$ given $F = f$ is 
\begin{align*}
    p(y|F = f):= \mathcal{N}\big(y  |  (f, D)_E, \sigma^2 I_N \big),
\end{align*}
for $y \in \bbR^N,  f \in E$, $(f, D)_E:= \big( (f, D_n)_E \big)_{n=1,\hdots, N}$ and $\mathcal{N}\big(\cdot  |  \mu, \Sigma \big)$ denotes the pdf of a Gaussian distribution on $\bbR^{N}$ with mean vector $\mu \in \bbR^N$ and covariance matrix $\Sigma \in \bbR^{N \times N}$. 

In the Bayesian paradigm, one updates their beliefs about $F$ after observing $Y\coloneqq(Y_1,.\hdots,Y_N)$ by combining the prior $F\sim \calN(0,C)$ with the likelihood $p(y|F=f)$ to form a posterior. This can be a delicate task since $E$ could be infinite dimensional. However, since in our scenario $E$ is a Banach space and the measures corresponding to $Y|F = f$ are all dominated by the Lebesgue measure on $\bbR^N$ with a jointly measurable map $(y,f) \in \bbR^N \times E \to p(y|f) \in \bbR$ , an infinite dimensional version of Bayes theorem applies \citep[Chapter 1.3]{ghosal2017fundamentals}. 



It states that a regular version \citep[chapter 8.3]{klenke2013probability} of the posterior measure exists, denoted $P^{F|Y}:\bbR^N \times \mathcal{B}(E) \to [0,\infty)$, $(y,A) \mapsto P^{F|Y=y}(A)$ and the measure $P^{F|Y=y}$ on $\mathcal{B}(E)$, which is the posterior measure of $F$ given $Y = y$, is dominated by the prior measure $P$ for any $y \in \bbR^N$ with Radon-Nikodym density $\frac{p(y|f)}{p(y)}$. 

What all these technicalities really mean is that for $A\in\calB(E)$
\begin{align}
    P^{F|Y=y}(A) = \int_A \frac{p(y|f)}{p(y)} dP(f),\label{eq: Bayes-Theorem}
\end{align}
where 
\begin{align*}
    p(y) & = \int_{E} p(y|F = f) dP(f) \\
    & = \mathcal{N}\big(y  |  (m, D)_E, C_{DD} + \sigma^2 I_N \big)
\end{align*}
with 
\begin{align}
    \big( (m,D)_E \big)_n &\coloneqq (m,D_{n}) \label{eq:m_e_def} \\
    (C_{DD})_{n,n'} &\coloneqq (C D_{n}, D_{n'})_{E}, \label{eq:C_ee_def}
\end{align}
for all $n,n'=1,\hdots,N$. 

This posterior measure $P^{F|Y = y}$ is a GM
since it is formed from the Gaussian likelihood $p(y|F=f)$ and Gaussian prior $F\sim\calN(m,C)$ (details in supplementary material section \ref{supp:sec_GRE}). Denote the mean and covariance operator of $P^{F|Y=y}$ by $\widetilde{m},\widetilde{C}$ respectively. As is usually the case with Bayesian techniques, the user is often not interested in the posterior measure itself but its pushforward through some prediction operation. 



We focus on the case of two linear maps $T,T'\in E^{*}$ since the general case of $S\in\bbN$ elements can be handled analogously. The posterior mean $\widetilde{m} \in E$ satisfies
\begin{align}
  \big(\widetilde{m}, T \big)_E =  (m, T)_E + C_{TD}(C_{DD}+\sigma^{2}I_{N})^{-1} y, \label{eq:GRE_post_mean}
\end{align}
for any $T \in E^*$ and the posterior covariance operator $\widetilde{C}: E^* \to E$ satisfies
\begin{align}
 (\widetilde{C}T &, T')_E\nonumber\\
 & =  (CT,T')_E - C_{TD} (C_{DD}+\sigma^{2}I_{N})^{-1} C_{DT'}\label{eq:GRE_post_var},
\end{align}
for all $T,T' \in E^*$, where  $C_{DD}$ as in \eqref{eq:C_ee_def} and $(C_{TD})_{1,n}  \coloneqq ( C T, D_{n})_{E}$ for $n=1,\hdots,N$ and $C_{DT'} = C_{T'D}^{\top} \in \bbR^{N \times 1}$. The proof for these statements is given in section \ref{supp:sec_GRE} of the supplementary materials.


In summary, given a prior GRE and some observed values $Y$ via some maps $\{D_{n}\}_{n=1}^{N}\subset E^{*}$ we can get a Gaussian posterior measure $P^{F|Y=y}$ for any $y \in \bbR^N$ on $E$. Two, or equivalently finitely many, linear functionals $T,T' \in E^*$ of $F$ under the posterior will follow a multivariate Gaussian distribution and one can use \eqref{eq:GRE_post_mean} and \eqref{eq:GRE_post_var} to calculate the mean vector and the covariance matrix.

\subsection{Example: GPs with Continuous Paths}\label{subsec:example}

Before we give the promised example that links GP regression to GRE regression, we need to introduce some key results from functional analysis.

Let $\calX \subset \bbR^D$ be compact with Borel $\sigma$-algebra $\mathcal{B}(\calX)$ and $C(\calX,\bbR)$ the space of continuous functions from $\calX$ to $\bbR$ equipped with the standard supremum norm. Denote by $R(\calX)$ the space of finite regular signed measures over $\calX$ equipped with total variation norm \citep[Section 2.4]{rao1983theory}. The Riesz-Markov theorem \citep[Corollary 4.7.6]{rao1983theory,Royden2010} states that $C(\calX,\bbR)^{*} = R(\calX)$ in the sense that each $\mu\in R(\calX)$ gives an element of $C(\calX,\bbR)^{*}$ via $f\mapsto\int_{\calX}f(x)d\mu(x)$ and for every $T\in C(\calX,\bbR)^{*}$ there exists a unique $\mu\in R(\calX)$ such that $Tf = \int_{\calX}f(x)d\mu(x)$. For example the pointwise evaluation map $\pi_{x}(f) = f(x)$ corresponds to the Dirac measure $\delta_{x}$ based at $x$ fo we write $\pi_{x} = \delta_{x}$. For a covariance operator $C\colon C(\calX,\bbR)^{*}\rightarrow C(\calX,\bbR)$ we set $C\mu$ to be $CT$ where $T$ is the unique element of $C(\calX,\bbR)^{*}$ such that $Tf = \int_{\calX}f(x)d\mu(x)$. 


We now establish the connection between GREs and GPs when $E = C(\calX,\bbR)$. Let $G\sim GP(0,k)$ be a Gaussian process over a compact subset $\calX\subset \bbR^{d}$ with kernel $k$ and zero mean. Zero mean is for simplicity, non-zero can be handled straightforwardly. 

Assume the GP has paths in $C(\calX,\bbR)$ with probability one. A standard result which provide a sufficient condition is the Kolmogorov continuity theorem \citep[Theorem 2.2.3]{Oksendal2003}, if $k$ is translation invariant then a condition regarding the decay of $k$ is provided by \citet[Corollary 1.5.5]{Adler2007} and a condition regarding the spectral measure \citep[Chapter 4.2.1]{Rasmussen2005} of $k$ by \citet[Page 22]{Adler2007}.

Following \citet[Example 2.4]{Lifshits2012}, see also \citet{Rajput1972}, any GP with almost surely continuous paths can be identified with a GRE $F$ taking values in $E = C(\calX,\bbR)$, denote the corresponding GM by $P$. The covariance operator $C$ of $F$ is given as
\begin{align}
    C\nu(\cdot) & = \int_{\calX}k(\cdot,x') d\nu(x') \label{eq:GRE_CovOp} \\
    (C\nu,\mu)_{E} & = \int_{\calX}\int_{\calX}k(x,x')  d\nu(x')d\mu(x),\label{eq:GRE_C}
\end{align}
for $\mu,\nu \in R(\calX)$. Using the identification of pointwise evaluation and Dirac measures mentioned above
\begin{align*}
    \text{Cov}_{P}[F(x),F(x')] &= \text{Cov}[(F,\delta_x)_E,(F,\delta_{x'})_E]  \\
    &=(C\delta_x,\delta_{x'})_{E} \\
    &= \int \int k(t,t') d\delta_x(t) d\delta_{x'}(t') \\
    &= k(x,x'),
\end{align*}
for any $x,x' \in \calX$ as expected. 

In standard GP regression one observes corrupted pointwise information about the unknown function at a collection of points $X = \{x_{n}\}_{n=1}^{N}\subset\calX$. So in the notation of the previous subsection $D_{n}= \delta_{x_{n}}$ is the map through which our observations are viewed.

Suppose we want to make a prediction at two new points $x,x'\in\calX$. This corresponds to the measures $T = \delta_x$ and $T' = \delta_{x'}$ and we know that $ \big (F(x), F(x') \big) |Y=y$ is multivariate Gaussian. From \eqref{eq:GRE_post_mean} we calculate the mean as
\begin{align*}
   \widetilde{m}(x) = (\widetilde{m},\delta_x)_{E} =k_{xX}(k_{XX}+\sigma^{2}I_{N})^{-1}y,
\end{align*}
and similarly for $m(x')$. Furthermore the covariance between $F(x)$ and $F(x')$ under the posterior is given by formula \eqref{eq:GRE_post_var} as
\begin{align*}
    (\widetilde{C} \delta_x, \delta_{x'} )_E = k(x,x') - k_{xX}(k_{XX}+\sigma^{2}I_{N})^{-1}k_{Xx'},
\end{align*}
where $k_{XX}$ is the matrix with $n,n'$-th entry $k(x_{n},x_{n'})$ and $k_{xX} = (k(x,x_{1}),\ldots,k(x,x_{N})) = k_{Xx}^{\top}$. This is the standard formula for the posterior mean and covariance of a GP given noisy pointwise observations \citep{Rasmussen2005}.

In summary, GRE regression on $E=C(\calX,\bbR)$ with observation functionals $D_n =\delta_{x_n}$, $n=1,\hdots,N$ recovers standard GPR.

\section{Variational Inference for Gaussian Random Elements}\label{sec:VGP}
The posterior expressions \eqref{eq:GRE_post_mean} and \eqref{eq:GRE_post_var} can have high computational cost since the matrix inverse term has naive cost $O(N^{3})$. To avoid this cost the variational approximation paradigm is often used where $P^{F|Y=y}$ is approximated by selecting a measure in a candidate family $\calQ\subset\calP(E)$ that is optimal according to some divergence. The earliest works on this method are \citet{titsias2009variational,hensman2013gaussian} and, as discussed in the introduction, this area has received a lot of attention and innovation in the GP community recently \citep{Hensman2018,Dutordoir20a,lemercier2021}. 

We will now describe variational inference for Gaussian random elements in an abstract Banach space $E$. Much is owed to the work of \citet{Matthews16}, which formulated the important equations in the context of Gaussian processes. The following presentation applies in generality of Banach spaces which is the most general derivation the authors are aware of. 

The following are desired properties of the family $\calQ$ of candidates to approximate the posterior
\begin{enumerate}
    \item Predictions involving any $Q\in\calQ$ must be computationally tractable and less expensive than the true posterior.
    \item $\calQ$ contains measures that give a good approximation for the true posterior $P^{F|Y=y}$.
    \item A measure of \textit{closeness} between each $Q\in\calQ$ and $P^{F|Y=y}$ must be tractable and cheap to evaluate.
\end{enumerate}
 
\textbf{Variational family:} The idea in the construction of $\calQ$ is to parameterise certain features of the target posterior with a multivariate Gaussian, the hope being that these features will represent the target posterior well even though there may be less features than the number of points observed. 

First choose $M \in \bbN$ elements from the dual $\{L_{m}\}_{m=1}^{M}\subset E^{*}$, the \emph{features}, and set $L = (L_{1},\ldots,L_{M})$, $L\colon E \to \bbR^M$. Denote $U_m\coloneqq (F,L_m)_E$, $m=1,\hdots,M$ and $U:=(U_1,\hdots,U_M)$. Define $Q^{L}\coloneqq \mathcal{N}({\mu,\Sigma})\in\calP(\bbR^{M})$ for some mean vector $\mu\in\bbR^{M}$ and covariance matrix $\Sigma\in \bbR^{M\times M}$. Starting with a $Q^{L}$ of this form we obtain a member of the approximating family $Q$ as
\begin{align*}
    Q(A) = \int_{A}\left(\frac{dQ^{L}}{dP^{L}}\circ L\right)(f) dP(f)
\end{align*}
where $A\in\calB(E)$ and $dQ^{L}/dP^{L}$ is the Radon-Nikodym derivative of $Q^{L}$ with respect to $P^{L}$ and $P^{L}$ is the law of $U$ under the prior equal to $\calN((m,L)_E, C_{LL})$. The idea is that $Q^{L}$ dictates the behaviour of $Q$ on the features $L$. 

While this formulation of a candidate $Q\in\calQ$ may seem obtuse, in Theorem  \ref{thm:posterior_measure} in the Supplement it is shown
\begin{align*}
    Q(A) = \int_{\bbR^M} \bbP\big( F \in A \big|U=u) dQ^{L}(u),
\end{align*}
which is used to deduce that each $Q\in\calQ$ is a GM with mean $m_{Q}$ 
\begin{align}
    (m_{Q}, T)_E  = (m,T)_E  + C_{LL}^{-1}\big(\mu -  (m, L)_E \big)C_{ L T} \label{eq:Q_mean}
\end{align}
for all $T \in E^*$ and covariance operator $C_{Q}$
\begin{align}
    \big( & C_{Q} T, T'\big)_E  \nonumber\\
    & = (C T, T')_E + C_{TL}^{} C_{ L L}^{-1}(\Sigma- C_{LL})C_{LL}^{-1}C_{L T'}\label{eq:Q_var}
\end{align}
for all $T,T' \in E^*$.

The variational parameters of this family are $\mu$, $\Sigma$ and potentially parameters that appear in the specification of the inducing features $L$. For ease of notation denote all of these parameters by $\eta$ and the $Q$ corresponding to this choice by $Q_{\eta}$.

\textbf{Measure of closeness:} The Kullback-Leibler (KL) divergence is the measure of closeness employed. For $P,Q\in\calP(E)$ with $Q$ absolutely continuous with respect to $P$, denoted $Q\ll P$
\begin{align*}
    KL(Q,P) = \int_{E}\log\left(\frac{dQ}{dP}(f)\right)dQ(f),
\end{align*}
and $KL(Q,P)$ is infinite if $Q$ is not absolutely continuous with respect to $P$. The variational parameter $\eta \in \Gamma$ is then selected by minimising the KL
\begin{align*}
   \eta^* \in \underset{\eta}{\argmin} \, KL\big( Q_\eta, P^{F|Y=y} \big).
\end{align*}
After $\eta^*$ has been determined, the posterior is approximated with $Q_{\eta^*}$ which we denote by $Q^{*}$ for ease of notation. 

The choice of $\calQ$ means one may rewrite the KL as
\begin{align*}
    &KL( Q, P^{F|Y=y} ) \\
    =&KL( Q , P) -  \bbE_Q[\log p(y|F)] + \log p(y)
\end{align*} 
where $\bbE_Q[\log p(y|F)]:=\int \log p(y|F = f)dQ(f)$, see Theorem \ref{thm:KL_divergence} in the Supplement.

\textbf{Optimisation:} Optimisation with respect to KL is performed by optimising the evidence based lower bound (ELBO) defined $\mathcal{L}:= - KL(Q, P) + \bbE_Q[\log p(y|F)]$.

The user can now optimise the parameters $\eta$, which we recall are $\mu,\Sigma$, analytically to obtain the optimal $Q^{*}\in\calQ$ within the candidate family. The parameters $\mu,\Sigma$ have a closed form expression and cost $\calO(NM^2+M^3)$ \citep{titsias2009variational}. These optimal choices for $\mu,\Sigma$, given fixed $L$, are given in Theorem \ref{thm:KL_divergence} in the Supplement. The resulting optimal mean and covariance operators, denoted $m_{Q^{*}}$ and $C_{Q^{*}}$ satisfy
\begin{align}
    (m_{Q^*}&,T)_E   \nonumber\\
    & =   C_{TL}  \big( \sigma^2 C_{LL} +  C_{LD} C_{DL} \big)^{-1} C_{LD} y \label{eq:m_Q_optimal}  \\
    (C_{Q^*}&T,T')_E \nonumber\\
    & = (C T, T')_E - C_{TL}^{} C_{ L L}^{-1} C_{L T'} \nonumber\\
    & + C_{TL}^{}  \big( C_{LL} + \frac{1}{\sigma^2} C_{LD} C_{DL} \big)^{-1} \big)C_{L T'}, \nonumber
    \end{align}
for all $T,T' \in E^*$. See Theorem \ref{thm:KL_divergence} in the Supplement for a proof. 

Alternatively, a user could numerically optimise $\eta$ using a factorisation of $\mathcal{L}$ over $N$ to make use of batch size optimisation \citep{hensman2013gaussian}. This leads to complexity $\mathcal{O}(N_B M^2+M^3)$, where $N_B\in \bbN$ is the batch size. 
 
The factorised ELBO is normally used for really large data sets and in this case the bottleneck is the inversion of $C_{LL}$ which causes the $\calO(M^{3})$ complexity term. Therefore it is vital for practitioners to choose $L$ which result in $C_{LL}$ being easy to invert, for example making $C_{LL}$ be diagonal. 

The next section investigates common choices of $L$ in the literature and derives a unifying perspective using GREs, crucial for understanding the different choices.

\section{Functional Analysis View of Inducing Features}
In this section several variational approaches are recovered within the GRE framework. The goal is obtaining a unified perspective and greater generality of the derivations.


As described in Section \ref{subsec:example}, the starting point for our analysis is a GRE $F\sim \calN(m,C)$ in $E = C(\calX,\bbR)$ whose corresponding measure on $E$ is denoted $P$. This GRE is then conditioned upon corrupted pointwise observations $\{x_{n}\}_{n=1}^{N}\subset\calX$ such that $Y_n=(F,\delta_{x_n})_E + \epsilon_n$ with $\epsilon_1,\hdots,\epsilon_N \sim \calN(0, \sigma^2)$.

Different choices of features $\{L_{m}\}_{m=1}^{M}\subset C(\calX,\bbR)^*$ shall lead to the original inducing point approach \citep{titsias2009variational}, inter-domain features \citep{lazaro2009inter} and variational Fourier features \citep{Hensman2018}. 

While the first two examples appear pedestrian the third crucially relies upon the GRE to reveal how Fourier features actually behave and under what conditions they are valid, greatly expanding their scope beyond the example in \citet{Hensman2018}.

Only the covariances  $C_{LL}, C_{LT}$ of the features are derived since this is all that is needed to compute the variational mean $m_Q$ and covariance operator $C_Q$ of the approximating $Q$, see \eqref{eq:Q_mean} and \eqref{eq:Q_var}. 

The prediction map $T$ will always be a single point evaluation at an arbitrary point $x\in\calX$. The case of other choices of $T$, in particular point evaluation at multiple locations, is straightforward. The term $C_{LL}$ is the bottleneck term in the computation, as discussed in the previous section. 

Heavy use is made of the fact that GREs are preserved under linear transformations. Indeed, let $\iota: E \to W$ be a bounded linear operator between the Banach spaces $E$ and $W$. If $F$ is a GRE in $W$ with mean $m$ and covariance operator $C$ then $\iota F$ is a GRE in $W$ with mean $\iota m$ and covariance operator $\iota C \iota^{*}$, where $\iota^*:W^* \to E^*$ is the adjoint operator \citep{Bogachev1998}.

\subsection{Inducing Points}

The inducing-points framework of \cite{titsias2009variational} chooses the inducing features as pointwise evaluations $L_{m} = \pi_{z_m}$, which corresponds to the measure  $\delta_{z_{m}}$ for some set of points $\{z_{m}\}_{m=1}^{M}\subset\calX$. 

Substituting this choice of $L$ into \eqref{eq:GRE_C} 
\begin{align*}
    (C_{LL})_{mm'} & = \text{Cov}_{P}(L_{m}F,L_{m'}F) \\
    & = \bbE_{P}[(F,\delta_{z_{m}})_{E}(F,\delta_{z_{m}})_{E}]\\ & = \int_{\calX}\int_{\calX}k(x,x')  d\delta_{z_{m}}(x)\delta_{z_{m'}}(x') \\
    & = k(z_{m},z_{m'}),
\end{align*}
and 
\begin{align*}
    (C_{TL})_{m} & = \text{Cov}_{P}(TF,L_{m}F) \\
    & = \bbE_{P}[(F,\delta_{x})_{E}(F,\delta_{z_{m}})_{E}]\\
    & = \int_{\calX}\int_{\calX}k(t,t')d\delta_{x}(t)  \delta_{z_{m}}(t')= k(x,z_{m}).
\end{align*}

Combining these calculations with \eqref{eq:Q_mean}, \eqref{eq:Q_var} we recover the same formula as seen in the original inducing point derivation \citep{titsias2009variational}. 

\subsection{Inter-domain Features}
The natural space to realise inter-domain features \citep{lazaro2009inter} is $L^{2}(\calX,\bbR)$ since they involve an inner product on $L^{2}(\calX,\bbR)$. To this end map the GRE $F$ that takes values in $C(\calX,\bbR)$ into $L^{2}(\calX,\bbR)$ via the canonical embedding $\iota: C(\calX,\bbR) \to L^2(\calX,\bbR)$, $f \mapsto f$. The adjoint operator of $\iota$ is given as $\iota^*: L^2(\calX,\bbR) \to R(\calX) $, $f \mapsto \int_{(\cdot)} f(x)  dx$, which can be easily verified. As explained in Section \ref{subsec:example} the space $C(\calX,\bbR)^{*}$ can be identified with the space of signed Borel measures on $\calB(\calX)$ which explains why $\iota^{*}(f)$ is an element of $R(\calX)$ and therefore accepts as input a set to integrate over. 

As $\iota$ is a bounded, linear operator from $C(\calX,\bbR)$ to $L^{2}(\calX,\bbR)$ we can use the result mentioned above to conclude $\iota F$ is a GRE in $L^{2}(\calX,\bbR)$ with covariance operator $\iota C\iota^{*}$ where
\begin{align}
    (\iota C \iota^*)(g) &=  C \big(\iota^*(g)\big) \label{eq:L2_equal}\\
    &= \int_{\calX} k(\cdot, x') g(x')  dx',\label{eq:cont_cov}
\end{align}
where \eqref{eq:L2_equal} is meant as equality in $L^{2}(\calX,\bbR)$ and \eqref{eq:cont_cov} is simply from the definition of the covariance operator of the GRE on $C(\calX,\bbR)$. By \eqref{eq:cont_cov} $\iota C\iota^{*}$ coincides with the well-studied integral operator associated with $k$ denoted $T_{k}$ \citep{Steinwart2012} and defined as $T_{k}f(s) = \int_{\calX}k(s,t)f(t)dt$. We will adopt the notation $T_k$ for $\iota C \iota^*$ from now on.

Inter-domain features can be written in our notation as 
\begin{align*}
    L_{m}F = \langle \iota F, g_m \rangle_{L^{2}(\calX,\bbR)} = \int_{\calX}F(x)g_{m}(x)  dx
\end{align*}
for some collection $\{g_{m}\}_{m=1}^{M}\subset L^{2}(\calX,\bbR)$. Using the definition of an adjoint operator this is equal to $L_{m}F = (F,\iota^{*}g_{m},)_{E}$ so $L_{m} = \iota^{*}g_{m}$. 

Substituting into \eqref{eq:GRE_C} 
\begin{align}
    (C_{LL}&)_{mm'}  \nonumber\\
    & = \text{Cov}_{P}(L_{m}F,L_{m'}F) \nonumber\\
    & = \bbE_{P}[(F,\iota^{*}g_{m})_{E}(F,\iota^{*}g_{m'})_{E}]\nonumber\\
    & = \int_{\calX}\int_{\calX}k(x,x')g_{m}(x)g_{m'}(x') dxdx',\label{eq:inter_CLL}
\end{align}
and
\begin{align}
    (C_{TL}&)_{m} \nonumber\\
    & = \text{Cov}_{P}(TF,L_{m}F) \nonumber\\
    & = \text{Cov}_{P}[(F,\delta_{x})_{E},(F,\iota^{*}g_{m})_{E}]\nonumber\\ & = \int_{\calX}\int_{\calX}k(t,x')g_{m}(x') d\delta_x(t) dx' = T_{k}(x),\label{eq:inter_CLT}
\end{align}
which agrees with the original derivation of inter-domain features \citep{lazaro2009inter}.

\subsection{Fourier Features}
Fourier features are defined as a reproducing kernel Hilbert space (RKHS) inner product between $F$ and trigonometric functions \citep{Hensman2018}. 

An RKHS is a Hilbert space of functions that, as the name suggests, is associated with a kernel. Namely, given a kernel $k$ the RKHS is the unique Hilbert space of functions mapping from $\calX$ to $\bbR$, which we denote $H_{k}$, such that $k(\cdot,x)\in H_{k}$ for all $ x\in\calX$ and $\langle f,k(\cdot,x)\rangle_{k} = f(x)$ for all $ f\in H_{k}$ and $x\in\calX$, where $\langle\cdot,\cdot\rangle_{k}$ denotes the inner product on $H_{k}$. This latter property is called the reproducing property. For more on the theory of RKHS consult \citet{Berlinet2004}.

The idea of Fourier features is to observe an inner product not in $L^{2}(\calX,\bbR)$, as was done in inter-domain features, but instead to observe an inner product in $H_{k}$. Namely, \citet{Hensman2018} use features that would be written in our framework as $L_{m}F = \langle F,g_{m}\rangle_{k}$ where $\{g_{m}\}_{m=1}^{M}$ are the first $M$ elements of the Fourier basis. However, it is well known that $F\notin H_{k}$ almost surely \citep{Lukic2001} therefore the aforementioned choice of $L$ \emph{cannot} be used without extra justification.

\citet{Hensman2018} provided justification in the particular case when $k$ is a Mat\'{e}rn kernel with certain parameters over $\mathcal{X}=\bbR$. An explicit form of the Mat\'{e}rn RKHS inner product is used and the core of the argument is the $g_{m}$ are ``very regular'' in an appropriate sense to compensate for the way that almost surely $F$ is not regular enough to be contained in $H_{k}$ \citep[Section 3.3.1]{Hensman2018}. 

A rigorous justification is now given for these type of features. The derivation is more general than just the one-dimensional Mat\'{e}rn case and a condition on which functions $g_m$ can be used instead of the Fourier basis is provided. 

For the rest of this section we will assume that $k$ is continuous. Recall the kernel integral operator $T_{k}$ defined above, the square root $T_{k}^{1/2}$ is an isometric isomorphism between $L^{2}(\calX,\bbR)$ and $H_{k}$ \citep[Theorem 4.51]{Steinwart2008} meaning that $T_{k}^{1/2}(L^{2}(\calX,\bbR)) = H_{k}$ and 
\begin{align}
    \langle T_{k}^{1/2}f,T_{k}^{1/2}g\rangle_{k} = \langle f,g\rangle_{L^{2}(\calX,\bbR)},\label{eq:RKHS_isom}
\end{align}
for all $f,g\in L^{2}(\calX,\bbR)$. The moral of this result is that $T_{k}^{1/2}$ bestows upon element of $L^{2}(\calX,\bbR)$ just enough regularity to be in $H_{k}$. This notion of adding regularity will tie into the notion of ``very regular'' employed by \citet{Hensman2018}. 

Define features $L_{m} = \iota^{*}f_{m}$ for some $\{f_{m}\}_{m=1}^{M}\subset L^{2}(\calX,\bbR)$ so 
\begin{align*}
    L_{m}F = (F,\iota^{*}f_{m})_{E} = \langle \iota F,f_{m}\rangle_{L^{2}(\calX,\bbR)}
\end{align*} where $\iota\colon C(\calX,\bbR)\rightarrow L^{2}(\calX,\bbR)$ is the inclusion operator used in the previous subsection. Then
\begin{align}
    ``\langle F,T_{k}f_{m}\rangle_{k}  = " & \langle T_{k}^{1/2}\iota F,T_{k}^{1/2}f_{m}\rangle_{k}\label{eq:quote_equal}\\
    & = \langle \iota F,f_{m}\rangle_{L^{2}(\calX,\bbR)}\label{eq:use_isom}\\
    & = (F,\iota^{*}f_{m})_{E} = L_{m}F.\label{eq:RKHS_inter}
\end{align}
The first expression is in quotes since $F\notin H_{k}$ almost surely \citep{Lukic2001} so the expression has no real meaning. We include it though since if one were to suspend reality then the equality in quotes would be valid since $T_{k}^{1/2}$ is self-adjoint on $H_{k}$ so it can be borrowed from $T_{k}f_{m}$. The second term in \eqref{eq:quote_equal} is well defined since $T_{k}^{1/2}$ maps from $L^{2}(\calX,\bbR)$ to $H_{k}$. The move to \eqref{eq:use_isom} is facilitated by the isometry \eqref{eq:RKHS_isom}. 

The main idea of what is happening in \eqref{eq:quote_equal} is $F$ is borrowing a $T_{k}^{1/2}$ from $T_{k}f_{m}$ to be able to live in the RKHS. This is happening explicitly in the calculations done by \citet{Hensman2018} in the Mat\'{e}rn case. Indeed, the way that $T_{k}f_{m}$ is the image of $f_{m}$ under \emph{two} applications of $T_{k}^{1/2}$, rather than just the standard one needed to be in the RKHS, is the mathematical explanation of the notion of ``very regular'' that was alluded to by \citet{Hensman2018} since it has had two portions of the regularity provided by $T_{k}^{1/2}$. 

It is interesting to see that RKHS inner product feature \eqref{eq:quote_equal} can be reduced to \eqref{eq:RKHS_inter} which is simply using $T_{k}f_{m}$ as an inter-domain feature. So using the inter-domain formula for $C_{LL}$ \eqref{eq:inter_CLL} gives
\begin{align}
    & ``\text{Cov}_{P}(\langle F,T_{k}f_{m}\rangle_{k},\langle F,T_{k}f_{m'}\rangle_{k}) = "\nonumber\\
    & (C_{LL})_{mm'} = \bbE_{P}[(F,\iota^{*}f_{m})_{E}(F,\iota^{*}f_{m'})_{E}]\nonumber\\
    &  = \langle T_{k}f_{m},f_{m'}\rangle_{L^{2}(\calX,\bbR)}\nonumber\\
    & = \langle T_{k}f_{m},T_{k}f_{m'}\rangle_{k},\label{eq:RKHS_inner}
\end{align}
where \eqref{eq:RKHS_inner} is using the isometry between $L^{2}(\calX,\bbR)$ and $H_{k}$ and the fact that $T_k^{1/2}$ is self-adjoint considered as operator on $L^2(\calX,\bbR)$. Similarly, using \eqref{eq:inter_CLT}
\begin{align}
    & ``\text{Cov}_{P}(\langle F,k(z,\cdot)\rangle_{k},\langle F,T_{k}f_{m}\rangle_{k}) = "\nonumber\\
    & (C_{TL})_{m}= \bbE_{P}[(F,\delta_{z})_{E}(F,\iota^{*}f_{m})_{E}]\nonumber\\
    & = T_{k}f_{m}(z) \label{eq:RKHS_x}
\end{align}

The connection to \citet{Hensman2018} is revealed once one sets $g_{m} = T_{k}f_{m}$ where $g_{m}$ are the features employed by \citet{Hensman2018}. In particular \eqref{eq:RKHS_inner} and \eqref{eq:RKHS_x} are equal to Equation $61$ and Equation $60$, respectively, in \citet{Hensman2018}. 

As was done by \citet{Hensman2018} it makes sense from a practical point of view to define $T_{k}f_{m}$ without explicitly choosing $f_{m}$. Our derivation shows this may be performed without dependence on kernel parameters as long as $g_{m}$ is in the range of $T_{k}$ for all the kernel parameters that could be considered. 

Our derivation is general enough to justify the use of Fourier features in \citet{Dutordoir20a}, where zonal kernels in more than one input dimension are used. Their choice of $g_{m}$ corresponds to eigenfunctions in the Mercer expansion of the kernel which is always in the image of $T_{k}$.

\section{Gaussian Random Elements and the Nyström Method}\label{sec:Connections}
In this section we demonstrate how the GRE framework reveals further connections between Gaussian process regression and the Nyström approximation for kernel Ridge regression (KRR). These connections have been known for a while and received some attention recently \citep{parzen1961approach,wahba1990spline,Kanagawa2018,wild2021connections}

Let $k$ be a kernel, $H_k$ the corresponding RKHS and $\{x_n,y_n\}_{n=1}^N \subset \calX \times \bbR$ be paired observations. In Nyström KRR \citep{williams2001using} we seek to minimise the empirical risk over a finite dimensional subspace $\mathcal{M} \subset H_k$
\begin{equation}
    \widehat{f} \coloneqq \underset{f \in \mathcal{M}}{\text{argmin}} \frac{1}{N} \sum_{n=1}^N \big( f(x_n) - y_n \big)^2 + \lambda \| f \|_k^2, \label{eq:KRR_Nystroem}
\end{equation}
where $\lambda>0$ is a regularisation parameter and $\widehat{f}$ is called the Nyström approximation over $\mathcal{M}$. 

The subspace $\mathcal{M}$ that is typically selected is $\text{span}\big(\{k(\cdot,z_m)\}_{m=1}^{M}\big)$, where $\{z_m\}_{m=1}^M \subset \calX$ are user chosen points, referred to as landmark points. Most research has focused on different sampling approaches for the landmark points to guarantee high quality approximations \citep{rudi2015less,musco2016recursive,li2016fast}. 

The GRE perspective facilitates a generalised view in which it becomes clear how the choice of features $L$ in variational GRE regression, see Section \ref{sec:VGP}, corresponds to the choice of subspace $\mathcal{M}$ in Nyström KRR. 

\begin{theorem}
Let $F \sim \mathcal{N}(0,C)$ be a GRE in $E=C(\calX,\bbR)$ with covariance operator $C$ as defined in \eqref{eq:GRE_CovOp} and assumed pointwise noisy data is observed as described in Section \ref{subsec:example}. Let $L_m = \mu_m$, where $\{\mu_m\}_{m=1}^M \subset R(\calX)$ be the features used in the variational approximation. Set $\mathcal{M}=\{C\mu_{m}\}_{m=1}^{M}$ where $C \mu_m = \int k(\cdot,x')d\mu_m(x')$ as the approximating family in the Nystr\"{o}m approximation. Then for $\sigma^{2} = N \lambda$ the Nystr\"{o}m KRR estimator $\widehat{f}$ in Section \ref{sec:Connections} is equal to the mean $m_{Q^{*}}$, given by \eqref{eq:m_Q_optimal}, of the optimal $Q^{*}$ from the variational family $\calQ$.
\end{theorem}

We give the proof and an explicit form of $\widehat{f}$ and $m_{Q^*}$ in Section \ref{sup:Connections} of the Supplement. The GRE view is essential to the proof. The result generalises the connection between Nystr\"{o}m KRR and inducing points outlined in \citet{wild2021connections} and opens the door to applying the theory of Nystr\"{o}m KRR error bounds to variational GPs to gain a better understanding of the latter's approximation properties. Vice versa, recent advances in variational GPR approaches, for example, variational Fourier features, could be leveraged in the context of KRR Nyström, as they simply correspond to a particular choice of $\mathcal{M}$.

\section{Conclusion}
We have outlined the GRE framework as a technical tool to provide a unified and generalising perspective to variational GP approximation. Along the way we have seen how existing choices of features, previously thought distinct, are in fact highly related and how they can be derived in wider settings than those currently employed. Finally, we related the posterior mean of variational GP approximations with a Nystr\"{o}m KRR approximation which offers a new lens to view the impact of different feature choices in variational GP approximations. 
\vfill

\subsubsection*{Acknowledgements}
The research of George Wynne was supported by an EPSRC Industrial CASE award [EP/S513635/1] in partnership with Shell UK Ltd.

\bibliographystyle{abbrvnat}
\bibliography{var.bib}

\newpage
\onecolumn
\begin{center}
    \textbf{\Large{Supplement}}
\end{center}

\section{Proofs of Section \ref{sec:GRER}: Gaussian Random Element Regression}\label{supp:sec_GRE}
Recall that the posterior measure in Section \ref{sec:GRER} satisfies $P^{F|Y=y}(A) = \int_A \frac{p(y|f)}{p(y)}dP(f)$ for all $A \in \mathcal{B}(E)$. 
\begin{theorem}\label{thm:posterior_measure}
\begin{enumerate}
\item For any measurable $h:E \to \bbR^S$
 \begin{align}
        P^{F|Y=y}\big( h^{-1}(B) \big) = \bbP( h(F) \in B |Y=y).
    \end{align}
 \item The posterior measure $P^{F|Y=y}$ is Gaussian with mean $\widetilde{m}$ satisfying
\begin{align}
  \big(\widetilde{m}, T \big)_E =  (m, T)_E + C_{TD}(C_{DD}+\sigma^{2}I_{N})^{-1} y, 
\end{align}
for any $T \in E^*$ and posterior covariance operator $\widetilde{C}: E^* \to E$ satisfying
\begin{align}
 (\widetilde{C}T, T')_E =  (CT,T')_E - C_{TD} (C_{DD}+\sigma^{2}I_{N})^{-1} C_{DT'} ,
\end{align}
for all $T,T' \in E^*$.
\end{enumerate}
\end{theorem}

\begin{proof}

As  $P^{F|Y=Y}$ is a regular version of the conditional probability of $F|Y=y$ it satisfies 
\begin{align}
    P^{F|Y=y}(A) = \bbP(F \in A|Y=y)
\end{align}
for any fixed $y \in \bbR^N$, $A \in \mathcal{B}(E)$ \citep[Definition 8.28]{klenke2013probability}. Hence for $A= h^{-1}(B)$ with $h:E \to \bbR^S$ measurable and $B \in \mathcal{B}(\bbR^N)$ 
\begin{align}
    P^{F|Y=y}( h^{-1}(B)) &= \bbP(F \in h^{-1}(B) |Y=y) \\
    &= \bbP(h(F) \in B |Y=y), \label{eq:GRE_post_fidi}
\end{align}
where the last line is using $\{\omega\in\Omega\colon F(\omega) \in h^{-1}(B) \} = \{ \omega\in\Omega\colon h(F(\omega)) \in B \}$, this proves the first statement.

For the second statement, we can use equation \eqref{eq:GRE_post_fidi} with the choice $h=(T,T')\colon E \to \bbR^2$ for two arbitrary $T,T' \in E^*$.  Recall that the random vector $X\colon \Omega \to \bbR^N$ is Gaussian if and only if $\alpha^{\top}X\colon \Omega \to \bbR$ is Gaussian for all $\alpha \in \bbR^N$. Using this we will show random vector $V\coloneqq\big( (F,T)_E, (F,T')_E, Y \big)$ in $\bbR^{N+2}$ is Gaussian. Take any $\alpha\in\bbR^{N+2}$ and set $\varphi = \sum_{n=1}^N \alpha_n D_n + \alpha_{N+1} T + \alpha_{N+2} T' \in E^*$, then
\begin{align}
    \alpha^{\top} V = \sum_{n=1}^{N+2} \alpha_n V_n = (F, \sum_{n=1}^N \alpha_n D_n + \alpha_{N+1} T + \alpha_{N+2} T')_E + \sum_{n=1}^N \alpha_n \epsilon_n = (F, \varphi)_E + \sum_{n=1}^N \alpha_n \epsilon_n,
\end{align}
is Gaussian for all $\alpha \in \bbR^{N+2}$ as it is sum of two independent Gaussian distributions, therefore $V$ is Gaussian.

The mean $\mu_{V}$ is given by
\begin{align}
    \bbE[\alpha^{\top} V] &= (m, \varphi)_E + 0 = \sum_{n=1}^N \alpha_n (m,D_n)_E + \alpha_{N+1} (m,T)_E + \alpha_{N+2} (m,T')_E \eqqcolon \alpha^{\top} \mu_V,
\end{align}
and by the characterising property of the covariance operator $C$ we get the covariance matrix $\Sigma_{V}$
\begin{align}
    \text{Cov}[\alpha^{\top} V, \alpha^{\top} V] = (C \varphi, \varphi)_E + \sigma^2 \sum_{n=1}^N \alpha_n^2 = \alpha^{\top} \Sigma_V \alpha ,
\end{align}
where the covariance matrix $\Sigma_V \in \bbR^{(N+2) \times (N+2)}$ is defined as 
\begin{align}
\Sigma_V \coloneqq
\begin{bmatrix}
 (CT,T)_E   & (CT,T')_E     & C_{TD}\\
 (CT',T)_E  & (CT',T')_E        & C_{T'D}\\
  C_{DT}   & C_{D T'}      & C_{DD} + \sigma^2 I_N.
\end{bmatrix}
\end{align}

We have showed $V \sim \mathcal{N}(\mu_V, \Sigma_V)$ and so using standard conditioning rules for multivariate Gaussians to condition on the last entry of $V$ it is clear that $h(F)|Y=y$ is Gaussian with the desired mean and covariance. 
\end{proof}

\section{Proofs of Section \ref{sec:VGP}: Variational Inference for Gaussian Random Elements }

Heavy use is made of the transformation rule for measures \citep[Theorem C]{halmos2013measure}. Let $(E, \mathcal{B}_E, \mu)$ be a measure space and $(W, \mathcal{B}_W)$ a measurable space. If $\mathcal{T}: E \to W$ and $\mathcal{G}: E \to [-\infty, \infty]$ are measurable, then 
\begin{align*}
   \int (\mathcal{G} \circ \mathcal{T})(x) \, d \mu(x) = \int \mathcal{G}(t) \, d \mu^\mathcal{T}(t),
\end{align*}
where the left hand-side exists if and only if the right hand-side exists.

\subsection*{Measures in the variational family are Gaussian measures}
Recall the definition of a measure in the variational family
\begin{align*}
    Q(A)\coloneqq \int_{A}\left(\frac{dQ^{L}}{dP^{L}}\circ L\right)(f) \, dP(f)
\end{align*}
for all $A \in \mathcal{B}(E)$ with $Q^L = \mathcal{N}(\mu, \Sigma)$ and $P^L=\calN((m,L)_E, C_{LL})$.

\begin{theorem}\label{thm:var_post_measure}
\begin{enumerate}
    \item For any $A \in \mathcal{B}(E)$
    \begin{align*}
     Q(A) = \int_{\bbR^M} \bbP\big( F \in A \big|U=u)d Q^{L}(u).
     \end{align*}
     \item The measure $Q$ is a Gaussian measure with mean $m_{Q}$ satisfying 
\begin{align}
    (m_{Q}, T)_E  = (m,T)_E  + C_{LL}^{-1}\big(\mu -  (m, L)_E \big), \label{eq:supp_mQ}
\end{align}
for all $T \in E^*$ and covariance operator $C_{Q}$ satisfying
\begin{align}
    \big( C_{Q} T, T'\big)_E  = (C T, T')_E + C_{TL}^{} C_{ L L}^{-1}(\Sigma- C_{LL})C_{LL}^{-1}C_{L T'}, \label{eq:supp_CQ}
\end{align}
for all $T,T' \in E^*$.
\end{enumerate}
\end{theorem}

\begin{proof}
Firstly, it suffices to prove the statement for sets of the form $A=T^{-1}(B)$, $B \in \mathcal{B}(\bbR)$, $T \in E^*$, as two measures on a Banach space coincide if and only if they coincide for all sets of this form. 

We want to apply the transformation rule for measures. To this end set $\mathcal{T} = (T, L): E \to \bbR^{M+1}, \mathcal{T}(F) = (T(F),L(F))$ and $V = T(F)$ and $\pi_L\big((x_1,\hdots,x_{M+1}) \big) = (x_2,\hdots,x_{M+1})$ for $x \in \bbR^{M+1}$. Clearly, $\{\omega\in\Omega\colon T(F(\omega)) \in B \} = \{\omega\in\Omega\colon \mathcal{T}(F(\omega)) \in B \times \bbR^{M} \}$ and  $\pi_L \circ \mathcal{T} = L$.

From this we use the transformation rule
\begin{align*}
    Q(\{T \in B\}) &= Q(\{ \mathcal{T} \in B \times \bbR^M\}) \\
    &= \int_{\mathcal{T} \in B \times \bbR^M} \left(\frac{dQ^{L}}{dP^{L}}\circ L\right)(f)dP(f)  \\
    &= \int_{\mathcal{T} \in B \times \bbR^M} \left(\frac{dQ^{L}}{dP^{L}}\circ \pi_L \circ \mathcal{T}  \right)(f) dP(f) \\
    &=  \int_{ B \times \bbR^M} \frac{dQ^{L}}{dP^{L}}(u) dP^{\mathcal{T}}(u,v)    \\
    &= \int_{ B \times \bbR^M} \frac{q(u)}{p(u)}   p(u,v) d(u,v),
\end{align*}
where we denote by $p(u)$ the probability density function (pdf) corresponding to $P^L$, by $q(u)$ the pdf corresponding to $Q^L$ and $p(u,v)$ the joint pdf corresponding to $P^{\mathcal{T}}$. By $p(u,v)=p(v|U=u) p(u)$ and an application of Fubini
\begin{align}
    Q(\{T \in B\}) &= \int_{\bbR^M} \Big( \int_{B} p(v|U=u)dv \Big) q(u) du \nonumber\\
    &=\int_{\bbR^M} \bbP( V \in B|U=u) d Q^L(u), \label{eq:Q_equiv_form}
\end{align}
which proves the claim.

For the second statement we maintain the notation $V=(F, T)_E$. We need to show that \eqref{eq:Q_equiv_form} is Gaussian for any choice of $T \in E^*$. It is well known that the conditional distribution $V|U=u$ can be written as
\begin{align}
    V|(U=u) \overset{\mathcal{D}}{=} (m,T)_E + C_{TL} C_{LL}^{-1}\big(  u - (m,L)_E\big) + W =:h(u,W), \label{eq:Gaussian_cond}
\end{align}
where $\overset{\mathcal{D}}{=}$ means equality in distribution and $W \sim \mathcal{N}\big(0, (C T, T)_E - C_{TL}^{} C_{ L L}^{-1} C_{LT} \big)$ independently of $U$. 

Since $h$ is linear in $U$, we know that $h(U,W)$ is Gaussian for $U\sim\mathcal{N}(\mu,\Sigma)$ and the mean and variance can easily be calculated as 
\begin{align}
    \bbE_Q[h(U,W)] &= (m,T)_E + C_{TL} C_{LL}^{-1}\big( \mu - (m,L)_E\big) \label{eq:mean1D_Q} \\
    \text{Cov}_Q[h(U,W)]&=\big( C_{Q} T, T\big)_E  
    = (C T, T)_E + C_{TL}^{} C_{ L L}^{-1}(\Sigma-C_{LL} )C_{LL}^{-1}C_{L T}. \label{eq:var1D_Q}
\end{align}
In other words $Q^T$ is Gaussian for any $T \in E^*$ and we conclude that $Q$ is a Gaussian measure.

To deduce $(C_Q T, T')_E$ for two arbitrary elements $T,T' \in E$ we reduce everything to the one-dimensional case.
For $\alpha, \beta \in \bbR$ let $\varphi =\alpha T + \beta T'$ then since $\varphi\in E^{*}$ and $F$ is Gaussian
\begin{align*}
     (F, \varphi)_E =\alpha (F,T)_E + \beta (F, T')_E ,
\end{align*}
is Gaussian. This proves, by definition, that $(F,T)_E$ and $(F,T')_E$ are jointly Gaussian under $Q$. The mean and variance of $(F, \varphi)_E,$ can be calculated from \eqref{eq:mean1D_Q} and \eqref{eq:var1D_Q}. Using standard linear algebra 
\begin{align*}
    \text{Cov}_Q[(F,T)_E, (F,T')_E] = (C T, T')_E + C_{TL}^{} C_{ L L}^{-1}(\Sigma- C_{LL})C_{LL}^{-1}C_{L T'}
\end{align*}
which shows $C_Q$ is as described in \eqref{eq:supp_CQ}.
\end{proof}

\subsection*{The Kullback-Leibler divergence is tractable}
In this section we show that the Kullback-Leibler divergence between the variational measure $Q$ and $P^{F|Y=y}$ can be re-written in a convenient form. This is well-known for finite dimensional Gaussians and has been done for the process view in \citet{Matthews16} but we have not seen such a derivation for Gaussian measures in Banach spaces so we include it here for completeness. 

First, recall the chain rule for Radon-Nikodym derivatives \cite[Chapter 32]{halmos2013measure}. Let $\mu, \nu$ and $\eta$ be $\sigma$-finite measures on the same measure space. If $\mu \ll \nu$ and $\nu \ll \eta$, then $\mu \ll \eta$ with Radon-Nikodym derrivative given as
\begin{align*}
    \frac{d \mu}{d \eta}(f) = \frac{d \mu}{d \nu}(f)  \frac{d \nu}{d \eta}(f)
\end{align*}
for $\eta$-almost every $f$.

\begin{theorem}\label{thm:KL_divergence}
\begin{enumerate}
    \item The Kullback-Leibler divergence satisfies
    \begin{align*}
    KL( Q, P^{F|Y=y} ) 
    &=KL( Q , P) -  \bbE_Q[\log p(y|F)] + \log p(y) \\
    &= -\mathcal{L} + \log p(y),
    \end{align*} 
    for any $y \in \bbR^N$.
    \item The ELBO is tractable and given as
    \begin{align*}
     \mathcal{L} &= \sum_{n=1}^N \Bigg( \log \mathcal{N} \big(y_n| (m,D_n)_E + C_{LL}^{-1} \big( \mu - (m,L)_E \big) C_{LD_n} , \sigma^2 \big) \\
      &-\frac{1}{2\sigma^2}  \Big( (C D_n, D_n)_E + C_{D_nL}^{} C_{LL}^{-1}(\Sigma - C_{LL})C_{LL}^{-1}C_{L D_n} \Big) \Bigg) \\
     &- KL\big( \mathcal{N}(\mu, \Sigma) , \mathcal{N}((F,m)_E, C_{LL})  \big)
    \end{align*}
      
    \item If the prior mean is zero, then the optimal values used in the variational family $\calQ$ for $\mu$ and $\Sigma$ are
    \begin{align*}
        \mu^* &= C_{LL} \big( \sigma^2 C_{LL} +  C_{LD} C_{DL} \big)^{-1} C_{LD} y \\
        \Sigma^*&= C_{LL} \big( C_{LL} + \frac{1}{\sigma^2} C_{LD} C_{DL} \big)^{-1} C_{LL}
    \end{align*}
    which then leads to the optimal mean and covariance satisfying
    \begin{align}
        (m_{Q^*},T)_E &  =   C_{TL}  \big( \sigma^2 C_{LL} +  C_{LD} C_{DL} \big)^{-1} C_{LD} y \label{eq:supp_m_Q_optimal}  \\
        (C_{Q^*}T,T')_E  &= (C T, T')_E - C_{TL}^{} C_{ L L}^{-1} C_{L T'} +
        C_{TL}^{}  \big( C_{LL} + \frac{1}{\sigma^2} C_{LD} C_{DL} \big)^{-1} \big)C_{L T'}, \label{eq:supp_C_Q_optimal}
    \end{align}
    for all $T,T' \in E^*$.
\end{enumerate}
\end{theorem}

\begin{proof}
For the proof of the first statement, by Bayes theorem we know that $P^{F|Y=y}$ is dominated by the prior $P$ with $\frac{d P^{F|Y=y}}{dP}(f) = \frac{p(y|F=f)}{p(y)}$ for any $y \in \bbR^N$. The reverse statement is also true, that $P$ is dominated by $P^{F|=y}$ for fixed $y \in \bbR^N$. This is a consequence of $f \mapsto \frac{p(y|F=f)}{p(y)} >0$ since then for any $A \in \mathcal{B}(E)$ with $P(A)>0$, $P^{F|Y=y}(A)=\int_A \frac{p(y|F=f)}{p(y)} dP(f) >0$, which the contrapositive of $P$ being dominated by $P^{F|Y=y}$. The Radon-Nikodym in this situation is given as $\frac{p(y)}{p(y|F=f)}$. Finally, by definition of $Q$ it is dominated by $P$ and $\frac{dQ}{dP}(f) = (\frac{dQ^L}{dP^L} \circ L \big)(f)$ for any $f \in E$. 

The chain-rule for Radon-Nikodym derivatives therefore tells us that $Q$ is dominated by $P^{F|Y=y}$ with 
\begin{align*}
    \frac{d Q}{d P^{F|Y=y}}(f) = \frac{d Q}{d P}(f) \frac{d P}{d P^{F|Y=y}}(f) =\big(\frac{dQ^L}{dP^L} \circ L \big) (f) \frac{p(y)}{p(y|F=f)}.
\end{align*}
This lets us rewrite the KL divergence as 
\begin{align*}
    KL(Q, P^{F|Y=y}) &= \int_{E} \log\left( \frac{d Q}{d P^{F|Y=y}}\right)(f)dQ(f) \\
    &= \int_{E} \log\left( \big(\frac{dQ^L}{dP^L} \circ L \big) (f) \frac{p(y)}{p(y|F=f)}  \right)dQ(f) \\
    &=\int_{E} \log \left( \big(\frac{dQ^L}{dP^L} \circ L \big) (f) \right)dQ(f) + \int_{E} \log \left( \frac{p(y)}{p(y|F=f)} \right)dQ(f) \\
    &= \int_{E} \log \left( \frac{dQ^L}{dP^L} (u) \right)dQ^L(u)
    - \int_{E} \log p(y|F=f)Q(f) + \log p(y) \\
    &= KL( Q^L, P^L) - \bbE_Q[\log p(y|F)] + \log p(y),
\end{align*}
which proves the first statement.

For the second statement, recall the ELBO is $\mathcal{L}=- KL(Q^L, P^L) + \bbE_Q[\log(y|F)]$. The KL term is clear, since $Q^L$ and $P^L$ are both Gaussian measures with the required mean and covariance. We therefore investigate the log-likelihood term now. Note that $Y_1,\hdots,Y_N$ are conditionally independent given $F=f$. The expected log-likelihood term therefore factorises as 
\begin{align*}
    \bbE_Q[\log p(y|F=f) ] = \sum_{n=1}^N \bbE_Q [ \log p(y_n|F) ].
\end{align*}
Setting $V=D_n(F)=(F,D_n)_E$ and noting that $p(y_n|F=f)$ depends on $F$ only through $D_n$, meaning $p(y_n|F=f) = p(y_n| V = D_n(f))$, we see
\begin{align*}
\bbE_Q [ \log p(y_n|F) ] &= \int_{E}\log p\big(y_n| V = D_n(f) \big) dQ(f) = \int_{E} \log p\big(y_n| V = v \big) dQ^{D_n}(v).
\end{align*}
Note that $Q^{D_n}$ is Gaussian with mean $\mu_V \coloneqq (m,D_n)_E + C_{LL}^{-1} \big( \mu - (m,L)_E \big) C_{LD_n}$ and variance $\sigma^2_V:= (C D_n, D_n)_E + C_{D_nL}^{} C_{ L L}^{-1}(\Sigma - C_{LL})C_{LL}^{-1}C_{L D_n}$ as $Q$ is a Gaussian measure. So using the parametric form of the pdf of a Gaussian 
\begin{align*}
    \bbE_Q[\log p(y_n|F)] &= \bbE_Q \big[ - \frac{1}{2\sigma^2} (y_n-V)^2 - \log( \sigma) - \frac{1}{2} \log(2 \pi) \big] \\
    &= \bbE_Q \big[  - \frac{1}{2\sigma^2} (y_n-\mu_V)^2 - \log( \sigma) - \frac{1}{2} \log(2 \pi)  -\frac{1}{2\sigma^2} (\mu_V-V_n)^2 \big] \\
    &=  \log \mathcal{N}(y_n| \mu_V, \sigma^2)  - \frac{1}{2 \sigma^2} \bbE_Q \big[ (\mu_v-V_n)^2 \big] \\
    &= \log \mathcal{N}(y_n| \mu_V, \sigma^2)  - \frac{1}{2 \sigma^2} \, \sigma_V^2,
\end{align*}
which proves the claim.

Finally, for the third statement, in Appendix A of \citet{titsias2009report} the optimal form of $\mu$ and $\Sigma$ are given. Note that we have the same objective function as \citet{titsias2009report} with the only difference that the kernel matrices $k_{nm}$ and $k_{mm}$ need to be replaced with the covariance matrices $C_{LD}$ and $C_{DD}$. Plugging in the optimal form for $\mu^*$ and $\Sigma^*$ into \eqref{eq:supp_mQ} and \eqref{eq:supp_CQ} gives rise to $m_{Q^*}$ and $C_{Q^*}$.
\end{proof}

\section{Proof of Section \ref{sec:Connections}: Connections between GRE Regression and KRR Nyström}\label{sup:Connections}

\begin{theorem}\label{thm:connection}
Let $F \sim \mathcal{N}(0,C)$ be a GRE in $E=C(\calX,\bbR)$ with covariance operator $C$ as defined in \eqref{eq:GRE_CovOp} and assumed pointwise noisy data is observed as described in Section \ref{subsec:example}. Let $L_m = \mu_m$, where $\{\mu_m\}_{m=1}^M \subset R(\calX)$ be the features used in the variational approximation. Set $\mathcal{M}=\{C\mu_{m}\}_{m=1}^{M}$ where $C \mu_m = \int k(\cdot,x')d\mu_m(x')$ as the approximating family in the Nystr\"{o}m approximation. Then for $\sigma^{2} = N \lambda$ the Nystr\"{o}m KRR estimator $\widehat{f}$ in Section \ref{sec:Connections} is equal to the mean $m_{Q^{*}}$, given by \eqref{eq:supp_m_Q_optimal}, of the optimal $Q^{*}$ from the variational family $\calQ$.
\end{theorem}

\begin{proof}
First of all note that $\mathcal{M}\subset H_{k}$, the RKHS of $k$. For a proof see \citep[Lemma 11.4]{ghosal2017fundamentals}. This means that the structure of $H_{k}$ can be leveraged to deduce $\widehat{f}$. Specifically, as every $f \in \mathcal{M}$ can be expressed as $f = \sum_{m=1}^M \alpha_m C \mu_m $ for some $\alpha \in \bbR^M$ we can solve the finite dimensional optimisation problem
\begin{align*}
    J(\alpha) \coloneqq \frac{1}{N} \sum_{n=1}^N \big( y_n - \sum_{m=1}^M \alpha_m C\mu_m(x_n) \big)^2 +  \lambda \| \sum_{m=1}^M \alpha_m C\mu_m\|_k^2,
\end{align*}
in $\alpha \in \bbR^M$ to find the KRR Nyström estimator. Expanding $J(\alpha)$
\begin{align*}
    J(\alpha)  &= \frac{1}{N} \sum_{n=1}^N y_n^2 - 2 \frac{1}{N} \sum_{n=1}^N \sum_{m=1}^M y_n \alpha_m g_m(x_n) +
    \frac{1}{N} \sum_{n=1}^N \sum_{m,m'=1}^M \alpha_m \alpha_{m'}  C\mu_m(x_n) C\mu_{m'}(x_n) \\
    &\quad  +\lambda \sum_{m,m'=1}^M \alpha_m \alpha_{m'} \langle C\mu_m , C\mu_{m'} \rangle_k \\
    &= \frac{1}{N} y^\top y - 2 \frac{1}{N} y^\top K_{X \mathcal{M}} \alpha + 
    \frac{1}{N} \alpha^\top K_{ \mathcal{M} X} K_{X \mathcal{M}} \alpha +
     \lambda \alpha^{\top} K_{\mathcal{M} \mathcal{M}} \alpha,
\end{align*}
where $(K_{\calM X})_{mn}= C\mu_{m}(x)$ and $(K_{\calM\calM})_{mm'}=\langle C\mu_{m}, C\mu_{m'} \rangle_{H_k}$ for $n=1,\hdots,N$ and $m,m'=1,\hdots,M$.
Standard rules for differentiation give
\begin{align*}
    J'(\alpha)&= -\frac{2}{N} K_{X \mathcal{M} } y +  \frac{2}{N}  K_{X \mathcal{M} }  K_{ \mathcal{M} X} \alpha + 2 \lambda K_{\mathcal{M} \mathcal{M}} \alpha \\
    J''(\alpha)&= \frac{2}{N} K_{X \mathcal{M}} K_{ \mathcal{M} X} + 2 \lambda K_{\mathcal{M} \mathcal{M}}.
\end{align*}
It is easy to see that $J'(\alpha)=0$ for $\alpha = (K_{\calM X} K_{X\calM} + N\lambda K_{\calM\calM} )^{-1} K_{\calM X} y$ and that $J''(\alpha)$ is positive definite. Hence the KRR estimator is given as 
\begin{align*}
     \widehat{f}(x) = \sum_{m=1}^M \alpha_m C\mu_{m}(x),
\end{align*}
with $\alpha = (\alpha_1,\hdots,\alpha_M)$ given as $\alpha = (K_{\calM X} K_{X\calM} + N\lambda K_{\calM\calM} )^{-1} K_{\calM X} y$.

On the other hand, from \eqref{eq:supp_m_Q_optimal} with $T= \delta_x$ and $D_n = \delta_{x_n}$
\begin{align*}
    m_{Q^*}(x) = \sum_{m=1}^M \beta_m  C \mu_m (x),
\end{align*}
with $\beta =  ( \sigma^2 C_{LL} + C_{LD} C_{DL})^{-1}  C_{LD}y$.  

The only thing left to show is $K_{\calM X} = C_{LD}$ and $K_{\calM\calM} = C_{LL}$. This is a consequence of the relationship between the so-called Cameron-Martin space of a GRE and the RKHS. In particular this exact equivalence is outlined by \citet[Page 316]{ghosal2017fundamentals} which completes the proof.
\end{proof}

\end{document}